\documentclass{article}
\usepackage{ijcai22}

\usepackage{times}
\usepackage{color,xcolor}
\usepackage{soul}
\usepackage{url}
\usepackage[hidelinks]{hyperref}
\usepackage[utf8]{inputenc}
\usepackage[small]{caption}
\usepackage{graphicx}
\usepackage{amsmath}
\usepackage{amsthm}
\usepackage{booktabs}
\usepackage{framed}
\usepackage{algorithm}
\usepackage{algorithmic}
\newtheorem{theorem}{Theorem}
\newtheorem{lemma}{Lemma}
\newtheorem{defn}{Definition}
\newtheorem{prop}{Proposition}
\newtheorem{corollary}{Corollary}

\title{Mechanism Design for Ad Auctions with Display Prices}

\author{
	Bin Li$^{1*}$\and
	Yahui Lei$^2$\\
	\affiliations
	{{$^1$School of Computer Science \& Engineering, Nanjing University of Science and Technology\\
			$^2$ Meituan Inc.\\}}
	\emails
	{cs.libin@njust.edu.cn, leiyahui@meituan.com}
}
\begin{document}
	\maketitle
	
	\begin{abstract}
		In many applications, ads are displayed together with the prices, so as to provide a direct comparison among similar products or services. The price-displaying feature not only influences the consumers' decisions, but also affects the advertisers' bidding behaviors.
		In this paper, we study ad auctions with display prices from the perspective of mechanism design, in which advertisers are asked to submit both the costs and prices of their products.
		We provide a characterization for all incentive compatible auctions with display prices, and use it to design auctions under two scenarios. In the former scenario, the display prices are assumed to be exogenously determined. For this setting, we derive the welfare-maximizing and revenue-maximizing auctions for any realization of the price profile. In the latter, advertisers are allowed to strategize display prices in their own interests. We investigate two families of allocation policies within the scenario and identify the equilibrium prices accordingly. 
		Our results reveal that the display prices do affect the design of ad auctions and the platform can leverage such information to optimize the performance of ad delivery. 
	\end{abstract}
	
	\section{Introduction}
	\footnote{$^*$ Corresponding Author.}
	Online advertising has been an indispensable part of modern advertising market. According to the newly released report of IAB \cite{iab:2021}, full-year online advertising revenue has reached \$189.3 billion in 2021. An important reason for wide-spread adoption of online advertisement comes from its high return on investment for advertisers, compared to other traditional marketing methods \cite{moran2005search}. As an efficient tool for deriving revenue, auctions are commonly used to allocate the display opportunities. Every day, tens of billions of ad auctions are conducted in real time to decide which advertisers' ads are shown, how these ads are arranged, and what the advertisers are charged. To date, online advertising platforms \cite{google:auto,microsoft:auto,facebook:auto} have developed various types of products for different types of advertisers, such as pay-per-mille or pay-per-impression (PPM), pay-per-click (PPC), and pay-per-action (PPA). In the classic ad auction setting, such as sponsored search, ads are presented in the form of hyper-links together with relevant keywords or well-designed creatives, serving as portals of advertisers' websites or products. Which ads are displayed depends on advertisers' bids and the relevance of their ads to the context \cite{VARIAN20071163}. However, in many real applications, like Temu or Ctrip, ads (or products) are displayed also together with the prices, e.g., it can be the per-night price of a room or the group purchase price of a commodity. The price-displaying feature brings two significant changes for the advertising system. On the one hand, the prices provide a direct comparison among similar products or services, which can easily influence the consumers' decisions. One the other hand, the price information also affects the advertisers' bidding behaviors and the efficiency of the underlying ad auctions \cite{Castiglioni2022EfficiencyOA}. 
	
	As the display prices has changed the bidding language and the way advertisers participate in the ad auction, fundamental investigation into mechanism design for auctions with display prices should be made. In this paper, we study how the presence of display prices affects ad auction design. In our model, advertisers are asked to submit both the costs and prices of their products and the advertising platform allocates the display opportunities and decides the charges based on the submitted information. Our model differs from the classic one in two ways. Firstly, rather than submitting a single bid for the display opportunities, we ask advertisers to submit both the costs and prices of their products. Secondly, in classic ad auctions, a convention, like a purchase, of an ad is exogenously determined and is independent of the advertisers' submission, while in our model conventions are essentially determined by the submitted display prices as the price information can dramatically affect the consumers' behaviors. Based on the framework of mechanism design, we carry out a systemic investigation on ad auctions with display prices. Specifically, we characterize all ad auctions that can incentivize truthful cost reports, and use the characterization to design auctions under two scenarios. In the former scenario, the display prices are assumed to be exogenously determined. For this scenario, we derive the welfare-maximizing and revenue-maximizing auctions for any realization of the price profile. In the latter scenario, the advertisers are allowed to strategize their display prices. For this setting, we investigate two families of allocation policies and compute the equilibrium price report accordingly. 
	%Besides, extensive experiments are conducted to evaluate the performances of the proposed auction mechanisms. 
	Our results show that the display prices do affect the design of ad auctions, and the advertising platform can leverage such information to further optimize the performance of ad delivery. 
	%In addition to the great industrial success, ad auctions have also attracted a lot of attention from the research community, and many research papers have been published on related topics, including click prediction, auction mechanism design, bidding optimization, etc.
	
	In addition to the great industrial success, ad auctions have attracted a lot of attention from the research community. Since Overture, for the first time, adopted the generalized first price auction mechanism in its sponsored search system in 1997 \cite{EDELMAN2007192,jansen2008}, many researchers from economics, computer sciences, management science, etc. have been working on different aspects of ad auctions for decades.
	%ad auctions have received considerable attention from the research community. For maximizing the performance of advertising systems, researchers from economics, computer sciences, management science, etc. have been working on ad auctions for decades. 
	Some of them focus on studying the theoretical properties of the deployed ad auctions \cite{edelman2007,VARIAN20071163,wilkens2017}, while some others are devoted to design new auction mechanisms towards different scenarios and objectives \cite{LAFFONT1996181,chenchen2019,glo2021}. With the development of AI, it is even possible to design ad auctions automately \cite{agg2019,shen2017,yang2019aiads}, based on advanced machine learning techniques and massive transaction data. To the best of our knowledge, Castiglioni et al. \cite{Castiglioni2022EfficiencyOA} is the very first to study ad auctions with display prices. The authors study the allocation efficiency of two widely used auctions, namely VCG and GSP, in the presence of display prices, and analyze the Price of Anarchy (PoA) and the Price of Stability (PoS) in the direct and indirect realizations of these two auctions, respectively. In contrast, we focus on the counterpart and study new ad auctions from the perspective of mechanism design.
	
	The reminder of this paper is organized as follows. Section 2 presents the basic model of auction with display prices and defines some general concepts of an auction mechanism. Section 3 characterizes all truthful auctions with display prices. Following the characterization, Section 4 investigates the welfare-maximizing and revenue-maximizing auctions, under the assumption of non-strategic display prices. Section 5 studies two families of auction mechanisms in the general settings and Section 6 summarizes this work.
	
	\section{Preliminaries}
	Assume there is a set of advertisers, denoted by $N=\{1,2,\cdots, n\}$, who decide to advertise their products in an online selling platform. For each advertiser $i\in N$, let $c_i\in [\underline{c}_i,\overline{c}_i]$ denote the cost (or type) of her product, which is assumed private information and derived from a distribution $\mathcal{C}_i$, and let $p_i$ denote the display price that $i$ sets for her product. As the display prices provide a direct comparison among similar products or services, for simplicity we assume the consumers' behaviors are determined by the products' display prices.
	%the consumers' behaviors, like clicking on an ad or purchasing a product, is determined by the prices. In other words, 
	%In addition to $c_i$, the selling platform asks each advertiser to submit her display price to allocate the ad slots.
	%In this work, the purchasing behaviors of the consumers are assumed to be determined by the display prices
	%the conversion rate of $i$'s product is determined by the settled display price.
	%The advertising platform asks each advertiser $i\in N$ to submit a cost and a display price of her product in order to allocate the display opportunities. For convenience, let $c_i$ denote the true cost (or type), which is private information derived from a distribution $\mathcal{C}_i$, and $p_i$ denote the display price that $i$ sets for her product. 
	%Given an advertiser $i$, the conversion rate of her product is essentially determined by the settled display price. 
	Formally, let $\lambda_i: \mathcal{R}\rightarrow [0, 1]$ denote the conversion rate function of $i$'s product. That is, $\lambda_i(p_i)$ represents the probability that a conversion, like a purchase, is acquired under display price $p_i$. Given an advertiser $i$, the expected value when her product is displayed on the selling platform can be formulated as $v_i(c_i, p_i)=(p_i-c_i)\lambda_i(p_i)$. The selling platform runs an auction to allocate the display opportunities. Besides the product cost, each advertiser is asked to report her display price to the auction.
	%As an efficient tool of deriving revenue, the selling platform often runs an auction to allocate the display opportunities (or ad slots) based on all advertisers' products costs and display prices. 
	Since $c_i$ is private information, advertiser $i$ can cheat the auction to benefit herself. Accordingly, let $(c_i', p_i)$ denote $i$'s report, where $c_i'$ is the reported cost and $p_i$ is the reported display price. For convenience, we use ${\bf c}'$ and ${\bf p}$ to denote the reported type profile and display prices of all advertisers, respectively. In addition, let ${\bf c}'_{-i}$ and ${\bf p_{-i}}$ be the reported type profile and reported prices of all advertisers except $i$, i.e., ${\bf c}' = (c_i', {\bf c}'_{-i})$ and ${\bf p} = (p_i, {\bf p}_{-i})$. The formal definition of auction mechanisms with display prices is given below.%We say ${\bf p}$ is a feasible price vector if the reported price $p_i$ is no less than the reported cost $c_i'$ for each advertiser $i$. In the following contents, we only consider feasible reported prices.
	
	\begin{defn}\label{int_auction}
		An \emph{auction mechanism} $\mathcal{M}=(\pi,x)$ consists of an allocation policy $\pi=\{\pi_i\}_{i\in N}$ and a payment policy $x=\{x_i\}_{i\in N}$, where $\pi_i: \mathcal{R}_{+}^{n+1}\rightarrow \{0, 1\}$ and $x_i:\mathcal{R}_{+}^{n+1}\rightarrow \mathcal{R}$ are the allocation and payment functions for $i$, respectively. %where $\mathcal{C}$ denotes the type space of all advertisers.
	\end{defn}
	Given all advertisers' reports $({\bf c}', {\bf p})$, $\pi_i({\bf c}', {\bf p})$ indicates whether or not advertiser $i$ wins the slot and $x_i({\bf c}', {\bf p})$ denotes the amount each advertiser $i$ pays to the platform. For advertiser $i$ with a report $(c_i', p_i)$, her utility function under $(\pi,x)$ is quasi-linear and is defined in the following:
	\begin{equation}
		u_i\big(c_i, {\bf c}^\prime, {\bf p}, (\pi,x)\big)=v_i(c_i, p_i)\pi_i({\bf c}', {\bf p})-x_i({\bf c}', {\bf p}). \label{uti_fun}
	\end{equation}
	%In addition, given an auction mechanism $(\pi,x)$, the optimal report is defined in the following:
	%\begin{equation}
	%	(p_i^*, c_i^*)=\arg\max_{\theta'}u_i\big(\theta_i, {\bf \theta}^\prime, (\pi,x)\big).
	%\end{equation}
	%Note that given an auction mechanism $\mathcal{M}$, advertiser $i$ will submit a report $(c_i', p_i)$ that maximizes her own (expected) utility. Therefore, one can view $i$'s report as the best reply against $\mathcal{M}$.
	
	Given an auction mechanism $\mathcal{M}=(\pi,x)$, the social welfare obtained in $({\bf c}', {\bf p})$, denoted by $W({\bf c}', {\bf p}, \mathcal{M})$, is defined as the total utilities of all agents (including the platform), which can be simplified as $W({\bf c}', {\bf p}, \mathcal{M})=\sum_{i\in N} v_i(c_i, p_i)\pi_i({\bf c}', {\bf p})$. We say an auction mechanism is efficient with reported prices (EF-RP) if for all reports $({\bf c}', {\bf p})$ it maximizes $W({\bf c}', {\bf p}, \mathcal{M})$.
	\begin{defn}\label{efficient_allocation_p}
		An auction mechanism $\mathcal{M}$ is \emph{efficient with reported prices (EF-RP)} if for all reports $({\bf c}', {\bf p})$
		\begin{equation}
			\mathcal{M} \in {\arg\max}_{\mathcal{M}'} W({\bf c}', {\bf p}, \mathcal{M}').
		\end{equation}
	\end{defn}
	Let $\Pi_i({\bf c}') = \arg\max_{{\bf p'}, \pi'}\sum_{i\in N} v_i(c_i, p_i')\pi'_i({\bf c}', {\bf p'})$. We say an auction mechanism $\mathcal{M}$ is efficient (EF) if for all reports $({\bf c}', {\bf p})$ it maximizes $W({\bf c}', {\bf p}, \mathcal{M})$ and ${\bf p}\in\Pi_i({\bf c}')$. 
	\begin{defn}\label{efficient_allocation}
		An auction mechanism $\mathcal{M}$ is \emph{efficient (EF)} if for all reports $({\bf c}', {\bf p})$
		\begin{equation}
			\mathcal{M} \in {\arg\max}_{\mathcal{M}'} W({\bf c}', {\bf p}, \mathcal{M}'),
		\end{equation}
		and the reported prices ${\bf p}\in \Pi_i({\bf c}')$.
	\end{defn}
	In other words, the EF-RP property only asks the mechanism to maximize the social welfare with every reports and the reported prices may not belong to $\Pi_i({\bf c}')$, and the EF property requires the mechanism to maximize the social welfare with every reports and the reported prices  must lie in $\Pi_i({\bf c}')$.
	%Essentially, EF and EF-RP are two different concepts. To see this, recall that the EF mechanism requires that the ad slot is always allocated to the advertiser with maximum gain and in the meanwhile the winner also chooses the optimal price $p_i^*$ in the maximum gain, i.e., $p_i^*=\arg\max_{p_i'}v_i(p_i', c_i)$, as the ultimate display price. However, the EF-RP mechanism only implements the efficient allocation for reported prices (these prices may not the optimal one). 
	Clearly, if an auction mechanism is EF, it is also EF-RP, but the reverse is not true.
	%By the revelation principle, we can restrict our attention on incentive compatible mechanisms w.l.o.g. %Incentive compatibility, or truthfulness, requires that  each agent to reveal her true type, no matter what the others do.
	%requires that each agent's utility is maximized when reporting the type truthfully, no matter what the others do.
	We next define several other properties that an auction mechanism should satisfy.
	
	\begin{defn}
		$\mathcal{M}=(\pi, x)$ is \emph{incentive compatible} (IC) if for all $i$, all $c_i$, all ${\bf p}$, and all ${\bf c}^\prime$,
		\begin{equation}
			u_i\big(c_i, (c_i,{\bf c}_{-i}^\prime), {\bf p}, \mathcal{M}\big) \geq u_i\big(c_i, (c_i',{\bf c}_{-i}^\prime), {\bf p}, \mathcal{M}\big).
		\end{equation}
	\end{defn}
	That is, incentive compatibility requires that submitting true product costs is a dominant strategy for all advertisers. Another important concept is called individual rationality. This property guarantees that each advertiser will not receive a negative utility when revealing her product cost truthfully. 
	%To ensure that all agents are willing to stay in the auction, the mechanism should also be individually rational.
	%In addition, to incentivize all agents to participate in the auction, the mechanism should also satisfy individual rationality property.
	%A mechanism satisfying IC is able to elicit advertisers' true valuations on the item, as well as the underlying structures of the intermediary market. 
	%Another important concept is called individual rationality.
	%in auction design
	%In an auction mechanism violating the IR property, agents can obtain negative utilities when revealing their true types, in which case quitting the auction is the best response. 
	%Individual rationality is also known as the participation constraint.
	
	\begin{defn}
		$\mathcal{M}=(\pi, x)$ is \emph{individually rational} (IR) if for all $i$, all $c_i$, all ${\bf p}$, and all ${\bf c}_{-i}^\prime$, 
		\begin{equation}
			u_i\big(c_i, (c_i,{\bf c}_{-i}^\prime), {\bf p}, \mathcal{M}\big) \geq 0.
		\end{equation}
	\end{defn}
	Note that if an auction mechanism violates the IR property, some advertisers may obtain negative utilities when revealing their true types, in which case quitting the auction is the best response. Therefore, individual rationality is also known as the participation constraint. Given advertisers' reports $({\bf c'}, {\bf p})$, the platform's revenue generated by an auction mechanism $\mathcal{M}$ is the sum of all advertisers' payments, denoted by $R({\bf c}', {\bf p}, \mathcal{M})=\sum_{i\in N}x_i({\bf c}', {\bf p})$. A major reason of adopting an auction to sell the display opportunities is to improve the performance of the advertising system, especially to increase the platform's revenue, so an auction with an external subsidy is not compelling.
	
	\begin{defn}
		$\mathcal{M}=(\pi, x)$ is \emph{weakly budget balanced} (WBB) if for all reports $({\bf c}', {\bf p})$, \begin{equation}
			R({\bf c}', {\bf p}, \mathcal{M}) \geq 0.
		\end{equation}
	\end{defn}
	
	In the following contents, we study auctions that satisfy IC, IR and other desired properties in the presence of display prices. We first provide a characterization for all IC and IR auction mechanisms, then using this characterization to design auctions towards different objectives. 

	\section{Characterizations of IC and IR Auctions with Display Prices}
	In this section, we characterize all IC and IR auctions in the presence of display prices.
	%We first investigate the conditions for an auction mechanism to be IC and IR, then formulate the seller's expected revenue based on the concept of Bayesian Nash equilibrium.
	%With the aid of the characterization, we design auctions towards different objectives in the next following sections.
	%We first investigate the conditions for an auction mechanism to be IC and IR, then formulate the seller's revenue based on the concept of Bayesian Nash equilibrium.
	We first present two necessary properties that an IC mechanism $(\pi, x)$ should hold, then show the two properties are also sufficient conditions for an auction mechanism to be IC.
	
	\begin{lemma}\label{non_inc_c}
		If an auction mechanism $(\pi,x)$ is IC, then $\pi_i$ is non-increasing in $c_i'$ for all $i$, all ${\bf p}$ and all ${\bf c}'_{-i}$.
	\end{lemma}
	\begin{proof}
		Consider two reported types $c_i^1$ and $c_i^2$ of advertiser $i$ with $c_i^1>c_i^2$. Incentive compatibility implies that for all ${\bf p}$ and ${\bf c}'_{-i},$
		\begin{align*}
			&v_i(c_i^1, p_i)\pi(c_i^1, {\bf c}'_{-i}, {\bf p})-x_i(c_i^1, {\bf c}'_{-i}, {\bf p})\ge \\
			&v_i(c_i^1, p_i)\pi(c_i^2, {\bf c}'_{-i}, {\bf p})-x_i(c_i^2, {\bf c}'_{-i}, {\bf p})
		\end{align*} and
		\begin{align*}
			&v_i(c_i^2, p_i)\pi(c_i^2, {\bf c}'_{-i}, {\bf p})-x_i(c_i^2, {\bf c}'_{-i}, {\bf p})\ge\\
			&v_i(c_i^2, p_i)\pi(c_i^1, {\bf c}'_{-i}, {\bf p})-x_i(c_i^1, {\bf c}'_{-i}, {\bf p}).
		\end{align*}
		Adding above two inequalities, we obtain
		\begin{small}
			\begin{align*}
				(v_i(c_i^1, p_i)-v_i(c_i^2, p_i))(\pi_i(c_i^1, {\bf c}'_{-i}, {\bf p})-\pi_i(c_i^2, {\bf c}'_{-i}, {\bf p}))\ge 0.
			\end{align*}
		\end{small}
		Recall that $v_i(c_i, p_i)$ is decreasing in $c_i$, therefore the above inequality implies that
		$$\pi_i(c_i^1, {\bf c}'_{-i}, {\bf p})\le \pi_i(c_i^2, {\bf c}'_{-i}, {\bf p}).$$
		That is, $\pi_i$ is non-increasing with $c_i'$ in any IC auction.
	\end{proof}
	
	Lemma \ref{ic_subg} unfolds the interconnections of the payment policy and the allocation policy.  It shows that in any IC auction, the allocation policy determines the payment policy.
	%sub-gradient of an advertiser's utility function is essentially determined by the allocation policy.
	\begin{lemma}\label{ic_subg}
		If an auction mechanism $(\pi, x)$ is IC, then for all $i$, all ${\bf p}$ and all ${\bf c}'_{-i}$, $x_i(c_i, {\bf c}'_{-i}, {\bf p})$ can be formulated as
		\begin{small}
			\begin{align}
				&v_i(c_i, p_i)\pi_i(c_i, {\bf c}'_{-i}, {\bf p})-\lambda_i(p_i)\int_{c_i}^{\overline{c}_i}\pi_i(z, {\bf c}'_{-i}, {\bf p})dz\nonumber\\
				&- {U}_i({\bf c}'_{-i}, {\bf p}), \label{payment}
			\end{align}
		\end{small}
		where ${U}_i({\bf c}'_{-i}, {\bf p})$ is independent of $i$'s cost report.
		%	
		%	\begin{small}
		%		\begin{equation}
		%		\frac{\partial{u_i(c_i, (c_i, {\bf c}'_{-i}), {\bf p})}}{\partial{c_i}}=-\lambda_i(p_i)\pi_i(c_i, {\bf c}'_{-i}, {\bf p}).
		%		\end{equation}
		%	\end{small}
	\end{lemma}
	\begin{proof}
		Given any IC mechanism $\mathcal{M}=(\pi, x)$, based on Def. 4, we have that for all $i$, all ${\bf p}$ and all ${\bf c}'_{-i}$ the following equation must hold:
		\begin{small}
			\begin{align*}
				u_i(c_i, (c_i, {\bf c}'_{-i}), {\bf p})=\max_{c_i'}v_i(c_i', p_i)\pi_i(c_i', {\bf c}'_{-i}, {\bf p})-x_i(c_i', {\bf c}'_{-i}, {\bf p}).
			\end{align*}
		\end{small}
		%\footnote{The envelope theorem is a result about the differentiability properties of a parameterized optimization problem. As we change parameters of the objective, the envelope theorem shows that changes in the optimizer of the objective do not contribute to the change in the objective function \cite{Michael2001}.}
		By the envelope theorem, this equation is equivalent to the following condition:
		\begin{small}
			\begin{align}
				&\frac{\partial{u_i(c_i, (c_i, {\bf c}'_{-i}), {\bf p})}}{\partial{c_i}} \nonumber\\
				&=	\frac{\partial{(v_i(c_i', p_i)\pi_i(c_i', {\bf c}'_{-i}, {\bf p})-x_i(c_i', {\bf c}'_{-i}, {\bf p}))}}{\partial{c_i}}|_{c_i'=c_i} \nonumber\\
				&=-\lambda_i(p_i)\pi_i(c_i, {\bf c}'_{-i}, {\bf p}). \label{uti_diff}
			\end{align}
		\end{small}
		Integrating both sides of formula (\ref{uti_diff}) over $[c_i, \overline{c}_i]$ on $c_i$, we have that $i$'s utility in an IC auction can be denoted by
		\begin{small}
			\begin{align*}
				u_i(c_i, (c_i, {\bf c}'_{-i}), {\bf p})={U}_i({\bf c}'_{-i}, {\bf p})+\lambda_i(p_i)\int_{c_i}^{\overline{c}_i}\pi_i(z, {\bf c}'_{-i}, {\bf p})dz,
			\end{align*}
		\end{small}
		where ${U}_i({\bf c}'_{-i}, {\bf p})=u_i(\overline{c}_i, (\overline{c}_i, {\bf c}'_{-i}), {\bf p})$ is independent of advertiser $i$'s cost report. 
		Based on formula (\ref{uti_fun}), we can further get that in an IC auction $(\pi, x)$ advertiser $i$'s payment $x_i(c_i, {\bf c}'_{-i}, {\bf p})$ can be formulated as 
		\begin{small}
			\begin{align*}
				&v_i(c_i, p_i)\pi_i(c_i, {\bf c}'_{-i}, {\bf p})-\lambda_i(p_i)\int_{c_i}^{\overline{c}_i}\pi_i(z, {\bf c}'_{-i}, {\bf p})dz\nonumber\\
				&- {U}_i({\bf c}'_{-i}, {\bf p}).
			\end{align*}
		\end{small}
	\end{proof}
	As $\lambda_i(p_i)$ and $\pi_i(c_i, {\bf c}'_{-i}, {\bf p})$ are non-negative, Lemma \ref{ic_subg} also suggests that an advertiser's utility is non-increasing with her product cost. 
	Our next result indicates that the above two conditions are also sufficient for an auction to be IC.
	\begin{theorem}\label{ic_cha}
		An auction mechanism $(\pi, x)$ is IC if and only if for all $i$, all ${\bf p}$ and all ${\bf c}'_{-i}$:
		\begin{itemize}
			\item [1.] $\pi_i$ is non-increasing in $c_i$;
			\item [2.] $x_i$ can be formulated as
			\begin{small}
				\begin{align*}
					&v_i(c_i, p_i)\pi_i(c_i, {\bf c}'_{-i}, {\bf p})-\lambda_i(p_i)\int_{c_i}^{\overline{c}_i}\pi_i(z, {\bf c}'_{-i}, {\bf p})dz\nonumber\\
					&- {U}_i({\bf c}'_{-i}, {\bf p}).
				\end{align*}
			\end{small}
		\end{itemize}
	\end{theorem}
\begin{proof}
	To prove this theorem, it suffices to prove that if an auction mechanism $\mathcal{M}=(\pi, x)$ satisfies Condition $1$ and $2$, then it is IC. 
	Given an advertiser $i$ with true cost $c_i$, to prove IC, we need to show that the inequality 
	\begin{small}
		\begin{align}
			u_i(c_i, (c_i, {\bf c}'_{-i}), {\bf p})\ge u_i(c_i, (c_i', {\bf c}'_{-i}), {\bf p}) \label{ic_inequ}
		\end{align}
	\end{small}	
	holds for all $c_i'\ne c_i$, all ${\bf p}$ and all ${\bf c}'_{-i}$. Plugging in the formula of $x_i(c_i, {\bf c}'_{-i}, {\bf p})$ and making simplification, the inequality (\ref{ic_inequ}) can be reformulated as
	\begin{small}
		\begin{align}
			\int_{c_i}^{\overline{c}_i}\pi_i(z, {\bf c}'_{-i}, {\bf p})dz&\ge \pi_i(c_i', {\bf c}'_{-i}, {\bf p})(c_i'-c_i)\nonumber \\&+\int_{c_i'}^{\overline{c}_i}\pi_i(z, {\bf c}'_{-i}, {\bf p})dz. \label{c_c_i}
		\end{align}
	\end{small}	
	
	{\bf Case 1:} If $c_i'>c_i$, the inequality (\ref{c_c_i}) is equivalent to 
	\begin{small}
		\begin{align}
			\pi_i(c_i', {\bf c}'_{-i}, {\bf p})(c_i'-c_i)\le \int_{c_i}^{c_i'}\pi_i(z, {\bf c}'_{-i}, {\bf p})dz,
		\end{align}
	\end{small}
	which is true under Condition $1$.
	
	{\bf Case 2:} If $c_i'<c_i$, the inequality (\ref{c_c_i}) is equivalent to 
	\begin{small}
		\begin{align}
			\pi_i(c_i', {\bf c}'_{-i}, {\bf p})(c_i-c_i')\ge \int_{c_i'}^{c_i}\pi_i(z, {\bf c}'_{-i}, {\bf p})dz,
		\end{align}
	\end{small}
	which is also true under Condition $1$.
\end{proof}
	Besides the IC property, another desired property is individual rationality, which requires all advertiser's utility to be non-negative when acting truthfully.

	\begin{theorem}\label{ir_cha}
		An IC auction mechanism $(\pi, x)$ is IR if and only if for all $i$, all ${\bf p}$ and all ${\bf c}'_{-i}$, 
		\begin{small}\begin{equation}
				{U}_i({\bf c}'_{-i}, {\bf p})\ge 0.
		\end{equation}\end{small}
	\end{theorem}
	\begin{proof}
		(``$\Rightarrow$'') For any IC auction mechanism $\mathcal{M}$, we have that advertiser $i$'s utility $u_i(c_i, (c_i, {\bf c}'_{-i}), {\bf p})$ is
		\begin{small}
			\begin{align*}
				{U}_i({\bf c}'_{-i}, {\bf p})+\lambda_i(p_i)\int_{c_i}^{\overline{c}_i}\pi_i(z, {\bf c}'_{-i}, {\bf p})dz.
		\end{align*}\end{small}
		Since $\lambda_i(p_i)\ge 0$ and $\pi_i(z, {\bf c}'_{-i}, {\bf p})$ is non-negative, we know that if ${U}_i({\bf c}'_{-i}, {\bf p})\ge 0$, then $u_i(c_i, (c_i, {\bf c}'_{-i}), {\bf p})\ge 0$.
		
		(``$\Leftarrow$'') If ${U}_i({\bf c}'_{-i}, {\bf p})< 0$ for some ${\bf p}$ and ${\bf c}'_{-i}$, then an advertiser with a report $(\overline{c}_i, p_i)$ will obtain a negative utility, which violates the IR property.
	\end{proof}
	
	Theorem \ref{ic_cha} and \ref{ir_cha} show the importance of the price information in auction design. In order to incentivize advertisers to report their product costs truthfully, the selling platform need take the display prices into consideration when designing ad auctions.
	As the platform acts as a profit-maximizing agent, for convenience's sake the term ${U}_i({\bf c}_{-i}, {\bf p})$ will be treated as zero hereafter for all $i$, all ${\bf c}_{-i}$ and all ${\bf p}$, without violating the IC and IR properties. 
	%The result shows that for incentivizing advertisers to report their product costs truthfully, the selling platform should take the display prices into consideration when designing ad auctions.
	
	To evaluate the performance of a given auction, we need to figure out how each advertiser submits her display price. If the display prices are exogenously determined, the techniques for traditional auctions can apply here. Otherwise, we need to compute the display prices in equilibrium at first. In the following, we first study auctions with non-strategic display prices, then turn to the general scenario where the advertisers are allowed to strategize their display prices.
	\section{Auction Design with Non-Strategic Display Prices}
	To keep and improve advertising performance, in some commerce activities advertisers seldom adjust their display prices frequently, especially when they also have offline brick-and-mortar shops. In this section, we investigate the auction design problem under the assumption of non-strategic display prices, i.e., the reported prices are exogenously determined and not affected by the underlying ad auctions.
	
	\subsection{Welfare-Maximizing Auctions with Reported Prices}
	Social welfare reflects the efficiency of a given allocation, which is defined as the summation of all agents' utilities. Based on Theorem \ref{ic_cha} and Theorem \ref{ir_cha}, we next propose an auction mechanism, called welfare maximizer with reported prices (short for WM-RP), to maximize the social welfare. As the reported display prices are assumed to be independent of the auctions, we adopt the concept of EF-RP (see Def. \ref{efficient_allocation_p}) to characterize social welfare maximization.
	
	%we adopt the concept of EF-RP (see Def. \ref{efficient_allocation_p}) here.
	%to characterize social welfare maximization.
	
	%Given an IC and IR auction $\mathcal{M}=(\pi, x)$, we know that each advertiser's payment is essentially determined by the allocation policy $\pi$
	%based on Theorem \ref{ic_cha} and \ref{ir_cha}. In other words, an allocation policy pins an auction mechanism. 
	
	\begin{framed}
		\noindent\textbf{Welfare Maximizer with Reported Prices (WM-RP)}\\
		\rule{\textwidth}{0.5pt}
		\begin{itemize}
			\item \textbf{Allocation policy:} Given reports $({\bf c}', {\bf p})$, allocate the ad slot to maximize $\sum_{i=1}^{N}v_i(c_i', p_i)\pi_i({\bf c}', {\bf p})$, break tie arbitrarily.
			\item \textbf{Payment policy:} For all advertiser $i\in N$, her payment $x_i({\bf c}', {\bf p})$ is defined below:
			\begin{itemize}
				\item  if $\pi_i({\bf c}', {\bf p})=0$, then $x_i({\bf c}', {\bf p})=0$;
				\item if $\pi_i({\bf c}', {\bf p})=1$, then $x_i({\bf c}', {\bf p})$
				%$x_i({\bf c}', {\bf p})=v^{(2)}({\bf c}', {\bf p})$.
				is defined as $$v_i(v_i^{-1}(v^{(2)}({\bf c}', {\bf p}),p_i), p_i),$$
			\end{itemize}
			where $v_i^{-1}$ is the inverse function of $v_i$ w.r.t. $c_i$ and $v^{(2)}({\bf c}', {\bf p})$ denotes the second highest value.
			%$v_i^{-1}$ is the inverse function\footnote{Since $v_i$ is non-increasing in $c_i$, the reverse function $v_i^{-1}$ is existing.} of $v_i$ and $v^{(2)}({\bf c}', {\bf p})$ is the second highest value under $({\bf c}', {\bf p})$.
		\end{itemize}
	\end{framed}
	
	In the WM-RP, the slot is allocated to the advertiser with the highest reported value, and only the winning advertiser pays to the platform. Next, we prove that the WM-RP maximizes the social welfare for any reported display price profile.
	
	%Since $v_i$ is non-increasing in $c_i$, then the winner's payment is identical to $v_i(v_i^{-1}(v^{(2)}({\bf c}', {\bf p}),p_i), p_i)=v^{(2)}({\bf c}', {\bf p})$. 
	%The reason why we use $v_i^{-1}(v^{(2)}, p_i)$ instead of $v^{(2)}$ is that 
	
	\begin{prop}\label{vm_result}
		The WM-RP is IC, IR and EF-RP.
	\end{prop}
	%\begin{proof}[Proof Sketch]
	%	One can check that the allocation policy is non-increasing in $c_i$ and the payment policy is consist with (\ref{payment}). So the WM-RP is IC and IR, and the EF-RP is satisfied by its allocation policy. 
	%\end{proof}
	\begin{proof}
		To prove this proposition it is sufficient to show that the WM-RP is IC and IR. Firstly, it is straightforward that the allocation policy is non-increasing in $c_i$ for all $i$, all ${\bf p}$ and all ${\bf c}_{-i}$, so the first condition of Theorem \ref{ic_cha} is satisfied. Secondly, we show that the payment policy is identical to (\ref{payment}) with $U_i({\bf c}'_{-i}, {\bf p})=0$. According to the allocation policy, the slot will be allocated to the one with the highest value $v_i(c_i', p_i)$. Given reports $({\bf c}', {\bf p})$, let $v^{(2)}({\bf c}', {\bf p})$ be the second highest value under $({\bf c}', {\bf p})$. For all losers $i$, $\pi_i({\bf c}', {\bf p})=0$ for all $c_i''\ge c_i'$ and therefore $i$'s payment $x_i({\bf c}', {\bf p})$ is zero according to (\ref{payment}). For the winner $i$, we know that her allocation $\pi_i({\bf c}', {\bf p})=1$ as long as 
		\begin{small}
			\begin{equation*}
				v_i(c_i', p_i)\ge v^{(2)}({\bf c}', {\bf p}),
			\end{equation*}
		\end{small}
		which is equivalent to the condition that 
		\begin{small}
			\begin{equation*}
				c_i'\le v_i^{-1}(v^{(2)}({\bf c}', {\bf p}), p_i),
			\end{equation*}
		\end{small}
		where $v_i^{-1}$ is the inverse function of $v_i$ w.r.t $c_i$ (recall that $v_i$ is non-increasing in $c_i$, so the reverse function $v_i^{-1}$ is existing). Let $U_i({\bf c}'_{-i}, {\bf p})=0$ for all $i$, all ${\bf c}'_{-i}$ and all ${\bf p}$. Then according to (\ref{payment}), $i$'s payment $x_i({\bf c}', {\bf p})$ is identical to
		\begin{small}
			\begin{align*}
				&v_i(c_i', p_i)\pi_i(c_i', {\bf c}'_{-i}, {\bf p}) -\lambda_i(p_i)\int_{c_i'}^{\overline{c}_i}\pi_i(z, {\bf c}'_{-i}, {\bf p})dz\\
				&=v_i(c_i', p_i) -\lambda_i(p_i)\int_{c_i'}^{v_i^{-1}(v^{(2)}({\bf c}', {\bf p}), p_i)}\mathrm{1}dz\\
				&-\lambda_i(p_i)\int_{v_i^{-1}(v^{(2)}({\bf c}', {\bf p}), p_i)}^{\overline{c}_i}0dz\\
				&=v_i(c_i', p_i) -\lambda_i(p_i)(v_i^{-1}(v^{(2)}({\bf c}', {\bf p}), p_i)-c_i')\\
				&=v_i(v_i^{-1}(v^{(2)}({\bf c}', {\bf p}),p_i), p_i).
			\end{align*}
		\end{small}
	\end{proof}
	As $v_i^{-1}$ is the inverse function of $v_i$, the winner's payment $v_i(v_i^{-1}(v^{(2)}({\bf c}', {\bf p}),p_i), p_i)$ is exactly $v^{(2)}({\bf c}', {\bf p})$$-$the second highest reported value.
	Recall that given any monotonic allocation policy, the only freedom of designing an IC auction is the choices of $U_i({\bf c}'_{-i}, {\bf p})$. Since IR requires $U_i({\bf c}'_{-i}, {\bf p})\ge 0$ and we set $U_i({\bf c}'_{-i}, {\bf p})$ to be zero in the WM-RP, the following result is straightforward.
	
	\begin{corollary}
		Among all IR, IC and EF-RP auctions, the WM-RP maximizes the platform's revenue.
	\end{corollary}
	
	Besides the allocation efficiency, another desiderata of the platform is revenue. We next investigate how to design auction mechanisms that maximize the platform's revenue.
	% under the assumption of non-strategic display prices.
	
	\subsection{Revenue-Maximizing Auctions with Reported Prices}
	%Given the distributions $\mathcal{C}_1, \mathcal{C}_2, \cdots, \mathcal{C}_n$ of advertisers' types, we next study the auction mechanisms optimizing the seller's expected revenue. 
	The following lemma gives a succinct description of advertiser's expected payment, which plays a key role in characterizing revenue-maximizing auctions.
	\begin{lemma}\label{expected_p}
		Given an IC and IR mechanism $\mathcal{M}$, reported prices ${\bf p}$ and a type profile of others ${\bf c}_{-i}$, the expected payment $\mathrm{E}_{c_i\sim \mathcal{C}_i}[x_i({\bf c}, {\bf p})]$ of advertiser $i$ is equal to:
		\begin{small}
			\begin{equation}
				\mathrm{E}_{c_i\sim \mathcal{C}_i}[\pi_i({\bf c}, {\bf p})\phi_i(c_i, p_i)],
			\end{equation}
		\end{small}
		where $\phi_i(c_i, p_i)=v_i(c_i, p_i)-\lambda_i(p_i)\frac{\mathcal{C}_i(c_i)}{\mathcal{C}'_i(c_i)}$ is called the virtual value of advertiser $i$.
	\end{lemma}
\begin{proof}
	Given an IC and IR mechanism $\mathcal{M}$, reported prices ${\bf p}$ and ${\bf c}_{-i}$, according to the proof of Theorem \ref{ic_cha} advertiser $i$'s expected payment $\mathrm{E}_{c_i\sim \mathcal{C}_i}[x_i({\bf c}, {\bf p})]$ is equal to 
	\begin{small}
		\begin{align*}
			&\int_{\underline{c}_i}^{\overline{c}_i}[v_i(y, p_i)\pi_i(y, {\bf c}_{-i}, {\bf p})]\mathcal{C}'_i(y)dy\\
			&-\int_{\underline{c}_i}^{\overline{c}_i}\mathcal{C}'_i(y)\int_{y}^{\overline{c}_i}\lambda_i(p_i) \pi_i(z, {\bf c}_{-i},{\bf p})dzdy.
		\end{align*}
	\end{small}
	Since $p_i$ is independent of advertiser $i$'s type $c_i$, we can change the order of integration and get that
	%(note that we cannot do this if $p_i$ is strategized)
	\begin{small}
		\begin{align*}
			&\int_{\underline{c}_i}^{\overline{c}_i}\mathcal{C}'_i(y)\int_{y}^{\overline{c}_i}\lambda_i(p_i) \pi_i(z, {\bf c}_{-i}, {\bf p})dzdy\\
			&=\int_{\underline{c}_i}^{\overline{c}_i}\lambda_i(p_i) \pi_i(z, {\bf c}_{-i}, {\bf p})\int_{0}^{z}\mathcal{C}'_i(y)dydz\\
			&=\int_{\underline{c}_i}^{\overline{c}_i}\lambda_i(p_i) \pi_i(z, {\bf c}_{-i}, {\bf p})\mathcal{C}_i(z)dz.
		\end{align*}
	\end{small}
	Therefore, $\mathrm{E}_{c_i\sim \mathcal{C}_i}[x_i({\bf c}, {\bf p})]$ can be formulated as 
	\begin{small}
		\begin{align*}
			&\int_{\underline{c}_i}^{\overline{c}_i}[v_i(y, p_i)\pi_i(y,{\bf c}_{-i}, {\bf p})\\
			&-\lambda_i(p_i)\pi_i(y, {\bf c}_{-i}, {\bf p})\frac{\mathcal{C}_i(y)}{\mathcal{C}'_i(y)}]\mathcal{C}'_i(y)dy\\
			&=\mathrm{E}_{c_i\sim \mathcal{C}_i}[\pi_i({\bf c}, {\bf p})\phi_i(c_i, p_i)],
		\end{align*}
	\end{small}
	where $\phi_i(y, p_i)=v_i(y, p_i)-\lambda_i(p_i)\frac{\mathcal{C}_i(y)}{\mathcal{C}'_i(y)}$.
\end{proof}

	Based on Lemma \ref{expected_p}, we can characterize the seller's expected revenue for any given display price profile.
	% wrt. $\mathcal{C}_1, \mathcal{C}_2, \cdots, \mathcal{C}_n$.
	\begin{theorem}\label{seller_rev}
		Given reported prices ${\bf p}$ and an IC and IR mechanism $\mathcal{M}$, the expected revenue $\mathrm{E}_{{\bf c}\sim \mathcal{C}}[\sum_{i=1}^{N}x_i({\bf c}, {\bf p})]$ of the platform is equal to the the expected virtual social welfare
		\begin{small}
			\begin{equation}
				\mathrm{E}_{{\bf c}\sim \mathcal{C}}[\sum_{i=1}^{N}\pi_i( {\bf c}, {\bf p})\phi_i(c_i, p_i)].
		\end{equation}\end{small}
	\end{theorem}
	\begin{proof}
		Take the expectation, with respect to ${\bf c}_{-i}$, of both sides of \begin{small}
			\begin{equation*}
				\mathrm{E}_{c_i\sim \mathcal{C}_i}[x_i({\bf c}, {\bf p})]=\mathrm{E}_{c_i\sim \mathcal{C}_i}[\pi_i({\bf c}, {\bf p})\phi_i(c_i, p_i)],
			\end{equation*}
		\end{small} we obtain that
		\begin{small}
			\begin{equation*}
				\mathrm{E}_{{\bf c}\sim \mathcal{C}}[x_i({\bf c}, {\bf p})]=\mathrm{E}_{{\bf c}\sim \mathcal{C}}[\pi_i( {\bf c}, {\bf p})\phi_i(c_i, p_i)].
			\end{equation*}
		\end{small}
		Apply linearity of expectations, we have that
		\begin{align*}
			&\mathrm{E}_{{\bf c}\sim \mathcal{C}}[\sum_{i=1}^{N}x_i({\bf c}, {\bf p})]=\sum_{i=1}^{N}\mathrm{E}_{{\bf c}\sim \mathcal{C}}[x_i({\bf c}, {\bf p})]\\
			&=\sum_{i=1}^{N}\mathrm{E}_{{\bf c}\sim \mathcal{C}}[\pi_i( {\bf c}, {\bf p})\phi_i(c_i, p_i)]\\
			&=\mathrm{E}_{{\bf c}\sim \mathcal{C}}[\sum_{i=1}^{N}\pi_i( {\bf c}, {\bf p})\phi_i(c_i, p_i)].
		\end{align*}
	\end{proof}
	Theorem \ref{seller_rev} indicates that the platform's expected revenue is identical to the expectation of the virtual social welfare. To maximize the platform's expected revenue, we can maximize the virtual social welfare pointwisely for any report $({\bf c}, {\bf p})$. Based on this observation, we now propose the virtual welfare maximizer with reported prices (short for VWM-RP).
	\begin{framed}
		\noindent\textbf{Virtual Welfare Maximizer with Reported Prices (VWM-RP)}\\
		\rule{\textwidth}{0.5pt}
		\begin{itemize}
			\item {\bf Allocation Policy:} Given reports $({\bf c}', {\bf p})$, allocate the ad slot to maximize $\sum_{i=1}^{N}\phi_i(c_i', p_i)\pi_i({\bf c}', {\bf p})$, break tie arbitrarily.
			\item \textbf{Payment policy:} For all advertiser $i\in N$, her payment $x_i({\bf c}', {\bf p})$ is defined below:
			\begin{itemize}
				\item  if $\pi_i({\bf c}', {\bf p})=0$, then $x_i({\bf c}', {\bf p})=0$;
				\item if $\pi_i({\bf c}', {\bf p})=1$, then $x_i({\bf c}', {\bf p})$ is defined as $$v_i(\phi_i^{-1}(\max\{\phi^{(2)}({\bf c}', {\bf p}),0\},p_i), p_i), $$
			\end{itemize}
			where $\phi_i^{-1}$ is the inverse function of $\phi_i$ w.r.t. $c_i$ and $\phi^{(2)}({\bf c}', {\bf p})$ denotes the second highest virtual value.
		\end{itemize}
	\end{framed}
	Different from the WM-RP, in the VWM-RP the allocation policy maximizes the virtual social welfare. Note that if all advertisers' virtual values are negative, the platform will not allocate the slot according to the allocation policy, i.e., the VWM-RP is not EF-RP. To implement the allocation policy of the VWM-RP, we need a regular condition on the distributions, which is defined below.
	\begin{defn}
		A distribution $\mathcal{C}_i$ is regular if the virtual value function $\phi_i(c_i, p_i)$ is non-increasing in $c_i$ for all $p_i$.
	\end{defn}
	Recall that $\phi_i(c_i, p_i)$ can be reformulated as $v_i(c_i+\frac{1}{\sigma_i(c_i)}, p_i)$, where
	$\sigma(c_i)=\mathcal{C}'_i(c_i)/\mathcal{C}_i(c_i)$ is the reverse hazard rate of $\mathcal{C}_i$. Since $v_i(\cdot,\cdot)$ is non-increasing with the first variable, a sufficient condition for regularity is that $\sigma(\cdot)$ is non-increasing. We next show if the distributions are regular, the VWM-RP will maximize the seller's revenue.
	\begin{theorem}
		Given a set of regular distributions $\mathcal{C}_1, \cdots, \mathcal{C}_n$, the VWM-RP is IC and IR, and maximizes the platform's revenue for any given set of display prices.
	\end{theorem}
	\begin{proof}
		%	Notice that when the distributions $\mathcal{C}_1, \cdots, \mathcal{C}_n$ are regular, $\phi_i^{-1}$ is existing for all advertisers $i\in N$.
		If $\mathcal{C}_1, \cdots, \mathcal{C}_n$ are regular, then the allocation policy of the VWM-RP is non-increasing in $c_i$. Following the proof of Proposition \ref{vm_result}, we can verify that the payment policy of the VWM-RP is consistent with (\ref{payment}), and therefore the VWM-RP is IC and IR according to Theorem \ref{ic_cha}. Since the VWM-RP maximizes the virtual welfare $\sum_{i=1}^{N}\pi_i({\bf c}', {\bf p})\phi_i(c_i', p_i)$ pointwisely for each report $({\bf c}', {\bf p})$, then based on Theorem \ref{seller_rev}, we know the VWM-RP maximizes the platform's revenue.
	\end{proof}
	If $\mathcal{C}_i$ is not regular, we can use the ``ironing technique'' to obtain a surrogate $\tilde{\mathcal{C}}_i$ \cite{myerson1981optimal}, which is regular and replaces $\mathcal{C}_i$ in the allocation policy.
	
	\section{Auction Design with Strategic Display Prices}
	In this section, we consider the general scenario where advertisers can report their display prices strategically. For this scenario, advertiser $i$ will choose a display price that maximizes her expected utility for any auction mechanism. 
	%Therefore, to evaluate the performance of an auction mechanism, the equilibrium prices should be first identified. 
	Since product costs are private information for all advertisers, Bayesian Nash Equilibrium (BNE) is a suitable solution concept for this setting, which is formally defined below.
	\begin{defn}\label{bne_price}
		Given an IC and IR auction $\mathcal{M}$, a strategy profile ${\bf p}^{\mathcal{M}}=(p_1^{\mathcal{M}},\cdots, p_n^{\mathcal{M}})$ is a Bayesian Nash Equilibrium if for all $i$, all $p_i'$ and all $c_i$,
		\begin{small}
			\begin{align*}
				&\mathrm{E}_{{\bf c}_{-i}\sim \mathcal{C}_{-i}}[u_i(c_i, {\bf c}, (p_i^{\mathcal{M}}(c_i), {\bf p}^{\mathcal{M}}_{-i}({\bf c}_{-i}))]\\
				&\ge \mathrm{E}_{{\bf c}_{-i}\sim \mathcal{C}_{-i}}[u_i(c_i, {\bf c}, (p_i', {\bf p}^{\mathcal{M}}_{-i}({\bf c}_{-i})))],
		\end{align*}\end{small}
		where ${\bf p}^{\mathcal{M}}_{-i}({\bf c}_{-i})=\{p_j^{\mathcal{M}}(c_j)\}_{j\in N\setminus\{i\}}$.
	\end{defn}
	In other words, a strategy profile ${\bf p}^{\mathcal{M}}$ is a BNE if no one can gain more utilities by unilaterally deviating from $p_i^{\mathcal{M}}(c_i)$. 
	Finding BNE in games is known to be a hard problem both analytically and computationally \cite{Victor2007} and previous works derived the BNE analytically only for the simplest auction settings \cite{krishna2009auction}. Recall that the price information enters into the allocation policy according to Theorem \ref{ic_cha}, hence it is impossible to obtain a closed form of the equilibrium prices for all auctions. 
	For tractability, in the following contents we study two special classes of allocation policies, namely the price-independent allocation policy and the affine maximizer allocation policy, in which the equilibrium prices are given analytically.
	
	\subsection{Price-Independent Allocation Policy}
	We first study price-independent allocation policy, where the slot is allocated without considering the reported display prices. It can apply to the circumstances where the display prices are unavailable before the auction or the advertisers tend to adjust their display prices dynamically after the auction.
	%Though the allocation is price-independent, the prices still get involved in the advertisers' payments for truthfulness. 
	The formal definition of price-independent allocation policy is given below.
	
	\begin{defn}[Price-Independent Allocation Policy]
		We say an allocation policy $\pi$ is price-independent (PI) if for all $i\in N$, and any two reports $({\bf c}',{\bf p}^1)$ and $({\bf c}',{\bf p}^2)$, $$\pi_i({\bf c}',{\bf p}^1)=\pi_i({\bf c}',{\bf p}^2).$$
	\end{defn}
	%Denote $p_i^*(c_i)=\arg\max_{p_i'}v_i(c_i, p_i')$, then we have the following result.
	Based on Theorem \ref{ic_cha}, we can easily derive the equilibrium prices for PI allocation policies.
	\begin{prop}\label{pi_all}
		Given any IC and IR mechanism $\mathcal{M}$ with a PI allocation policy, the price $$\overline{p}_i^{\mathcal{M}}(c_i)=\arg\max_{p_i'}\{\lambda_i(p_i')\}$$ forms the (dominant-strategy) equilibrium price for all $i$.
	\end{prop}
	\begin{proof}
		Given any IC and IR mechanism $\mathcal{M}$, according to Theorem \ref{ic_cha} advertiser $i$'s utility can be simplified as 
		$$u_i(c_i, (c_i, {\bf c}_{-i}), {\bf p})=\lambda_i(p_i)\int_{c_i}^{\overline{c}_i}\pi_i(z, {\bf c}_{-i}, {\bf p})dz.$$
		
		If the allocation policy is PI, then $\int_{c_i}^{\overline{c}_i}\pi_i(z, {\bf c}_{-i}, {\bf p})dz$ is independent of $p_i$. Therefore, the expected utility of advertiser $i$ is identical to \begin{align*}
			&\mathrm{E}_{{\bf c}_{-i}\sim \mathcal{C}_{-i}}[u_i(c_i, {\bf c}, {\bf p},\mathcal{M})]\\
			&=\mathrm{E}_{{\bf c}_{-i}\sim \mathcal{C}_{-i}}[\lambda_i(p_i)\int_{c_i}^{\overline{c}_i}\pi_i(z, {\bf c}_{-i}, {\bf p})dz]\\
			&=\lambda_i(p_i)\mathrm{E}_{{\bf c}_{-i}\sim \mathcal{C}_{-i}}[\int_{c_i}^{\overline{c}_i}\pi_i(z, {\bf c}_{-i})dz].
		\end{align*}
		To maximize the expected utility, $i$ can choose the price $p_i$ that maximize $\lambda_i(p_i)$, no matter what the others report.
		% for all reports $({\bf c}, {\bf p})$.
	\end{proof}
	Proposition \ref{pi_all} indicates that the equilibrium price $\overline{p}_i^{\mathcal{M}}(c_i)$ is a dominate strategy and is independent of the cost reports of advertisers. If the platform considers all PI allocation policies, then the following auction maximizes the revenue.
	
	\begin{framed}
		\noindent\textbf{Virtual Welfare Maximizer with Price-Independent Allocations (VWM-PIA)}\\
		\rule{\textwidth}{0.5pt}
		Given a set of convention rate functions $\{\lambda_i\}_{i\in N}$, compute the equilibrium price profile $\overline{{\bf p}}=\{\overline{p}_i\}_{i\in N}$, where $\overline{p}_i=\arg\max_{p_i'}\lambda_i(p_i')$.
		\begin{itemize}
			\item {\bf Allocation Policy:} Given reports $({\bf c}', {\bf p})$, allocate the ad slot to maximize $\sum_{i=1}^{N}\phi_i(c_i', \overline{p}_i)\pi_i({\bf c}', {\bf p})$, break tie arbitrarily.
			\item \textbf{Payment policy:} For all advertiser $i\in N$, her payment $x_i({\bf c}', {{\bf p}})$ is defined below:
			\begin{itemize}
				\item  if $\pi_i({\bf c}', {\bf p})=0$, then $x_i({\bf c}', {\bf p})=0$;
				\item if $\pi_i({\bf c}', {\bf p})=1$, then $x_i({\bf c}', {\bf p})$ is defined as $$v_i(\phi_i^{-1}(\max\{\phi^{(2)}({\bf c}', {{\bf p}}),0\},{p}_i), {p}_i).$$
			\end{itemize}
		\end{itemize}
	\end{framed}
	
	\begin{theorem}
		Given a set of regular distributions $\mathcal{C}_1, \cdots, \mathcal{C}_n$, the VWM-PIA is IC and IR, and maximizes the platform's revenue over all PI allocation policies.
	\end{theorem}
	\begin{proof}
		According to Proposition \ref{pi_all}, if the allocation policy is PI, then the equilibrium price is cost-independent. Hence, Lemma \ref{expected_p} and Theorem \ref{seller_rev} can extend to the PI allocation policies. Since the allocation policy of the VWM-PIA is PI and the cost distributions are regular, one can check that the payment policy of the VWM-PIA is consistent with (\ref{payment}), i.e., the VWM-PIA is IC and IR. In addition, as the allocation is PI, then each advertiser will submits $\overline{p}_i$ as her display price in equilibrium. Recall that the VWM-PIA maximizes the virtual welfare pointwisely for each report, then according to Theorem \ref{seller_rev} the VWM-PIA maximizes the revenue over all PI allocation policies.
	\end{proof}
	% Since $\overline{p}_i^{\mathcal{M}}(c_i)$ is independent of advertisers' reports, the results obtained in the last section, including the proposed welfare-maximizing and revenue-maximizing mechanisms, also apply here, where the reported display prices are fixed to be $(\overline{p}_1^{\mathcal{M}}(c_1), \cdots, \overline{p}_N^{\mathcal{M}}(c_N))$.
	
	\subsection{Affine Maximizer Allocation Policy}
	This section investigates another family of allocation policy, called affine maximizer. In this kind of allocation policy, the slot will be allocated to maximize the weighed and boosted social welfare. The platform can utilize such allocation policies to artificially increase the winning chance of advertisers with low values or outright ban certain outcomes in order to boost the revenue \cite{guo2017}. The formal definition of an affine maximizer allocation policy is given below.
	%In addition, a classic result by Roberts shows that under certain conditions, the only incentive-compatible auctions are the ones with affine maximizer allocation policies.
	
	\begin{defn}
		Given some advertiser weights ${\bf w}=(w_1, \cdots, w_n) \in \mathcal{R}_{+}^{N}$ and boosts ${\bf b}=(b_1, \cdots, b_n) \in \mathcal{R}^{N}$, an allocation policy $\pi$ is called an affine maximizer if for all reports $({\bf c}', {\bf p})$, the allocations are $$\pi({\bf c}', {\bf p})\in \arg\max_{\pi'({\bf c'},{\bf p})}\{b_{i}+\sum_{i\in N}w_iv_i(c_i', p_i)
		\pi_i'({\bf c'}, {\bf p})\}.$$
		% where $\pi_{i^*}'({\bf c}', {\bf p})=1$.
	\end{defn}
	Let $\psi_i(c_i', p_i)=b_i+w_iv_i(c_i', p_i)$ be the weighed and booted value of advertiser $i$, then an affine maximizer allocation policy will allocate the slot to advertiser $i$ with the maximum $\psi_i(c_i', p_i)$. Accordingly, ad auctions with an affine maximizer, called affine maximizer auction, is given below.
	%If the allocation policy is an affine maximizer, we can prove that the equilibrium prices has a simple form.
	\begin{framed}
		\noindent\textbf{Affine Maximizer Auction (AMA)}\\
		\rule{\textwidth}{0.5pt}
		Predefine a set of weights $\{w_i\}_{i\in N}$ and boosts $\{b_i\}_{i\in N}$.
		\begin{itemize}
			\item {\bf Allocation Policy:} Given reports $({\bf c}', {\bf p})$, allocate the ad slot to maximize $\sum_{i=1}^{N}\psi_i(c_i',p_i)\pi_i({\bf c}', {\bf p})$, break tie arbitrarily.
			\item \textbf{Payment policy:} For all advertiser $i\in N$, her payment $x_i({\bf c}', {{\bf p}})$ is defined below:
			\begin{itemize}
				\item  if $\pi_i({\bf c}', {\bf p})=0$, then $x_i({\bf c}', {\bf p})=0$;
				\item if $\pi_i({\bf c}', {\bf p})=1$, then $x_i({\bf c}', {\bf p})$ is defined as $$v_i(\psi_i^{-1}(\psi^{(2)}({\bf c}', {{\bf p}}),{p}_i), {p}_i),$$
			\end{itemize}
			where $\psi_i^{-1}$ is the inverse function of $\psi_i$ w.r.t. $c_i$ and $\psi^{(2)}({\bf c}', {{\bf p}})$ denotes the second highest weighted and boosted value.
		\end{itemize}
	\end{framed}
	\begin{prop}\label{aff_all}
		The AMA is IC and IR, and the price $$\tilde{p}_i^{\mathcal{M}}(c_i)=\arg\max_{p_i'}\{v_i(c_i, p_i')\}$$ forms the (dominant-strategy) equilibrium price for all $i$.
	\end{prop}
	\begin{proof}
		According to the definition of affine maximizer, we know that the allocation policy is non-increasing in $c_i$. In addition, we can further verify that the payment policy is consistent with (\ref{payment}), and thereby the AMA is IC and IR. Next, we show that $\tilde{p}_i^{\mathcal{M}}(c_i)=\arg\max_{p_i'}\{v_i(c_i, p_i')\}$ constitutes the equilibrium price for all $i$. Given any report $({\bf c}, {\bf p})$, by the payment policy the winner's payment in the AMA equals is  $$\lambda_i(p_i)[\psi_i^{-1}(\psi^{(2)}({\bf c}, {\bf p}), p_i)-c_i].$$ As $$\psi_i^{-1}(y, p_i)=p_i-\frac{y-b_i}{w_i\lambda(p_i)},$$ then winner $i$'s utility can be reformulated as $$u_i(c_i, (c_i, {\bf c}_{-i}), {\bf p})=v_i(c_i, p_i)-\frac{\psi^{(2)}({\bf c}, {\bf p})-b_i}{w_i}.$$
		%	Let $\psi_i^{-1}(y, p)$ be the reverse function of $\psi_i$ wrt. $c_i$. Given any IC and IR mechanism $\mathcal{M}$ with an affine maximizer allocation policy and any report $({\bf c}, {\bf p})$, advertiser $i$'s utility is given by 
		%	\begin{small}
		%		\begin{align*}
		%		&u_i(c_i, (c_i, {\bf c}_{-i}), {\bf p})=\lambda_i(p_i)\int_{c_i}^{\overline{c}_i}\pi_i(z, {\bf c}_{-i}, {\bf p})dz.
		%		\end{align*}
		%	\end{small}
		%	
		%	{\bf Case 1:} Bidder $i$ loses the item. For this case, we know that $\pi_i(z, {\bf c}_{-i}, {\bf p})=0$ for all $z\in [c_i, \overline{c}_i]$, and therefore her utility is zero. 
		%	
		%	{\bf Case 2:} Bidder $i$ wins the item. Denote by $\psi^{(2)}({\bf c}, {\bf p})$ the second highest weighed and boosted value under $({\bf c}, {\bf p})$, then $i$'s utility $u_i(c_i, (c_i, {\bf c}_{-i}), {\bf p})$ can be denoted by 
		%	\begin{small}
		%		\begin{align*}
		%		&\lambda_i(p_i)[\int_{c_i}^{\psi_i^{-1}(\psi^{(2)}({\bf c}, {\bf p}), p_i)}\pi_i(z, {\bf c}_{-i}, {\bf p})dz\\
		%		&+\int_{\psi_i^{-1}(\psi^{(2)}({\bf c}, {\bf p}), p_i)}^{\overline{c}_i}\pi_i(z, {\bf c}_{-i}, {\bf p})dz]\\
		%		&=\lambda_i(p_i)[\psi_i^{-1}(\psi^{(2)}({\bf c}, {\bf p}), p_i)-c_i].
		%		\end{align*}
		%	\end{small}	
		%	
		%	Since $$\psi_i^{-1}(y, p_i)=p_i-\frac{y-b_i}{w_i\lambda(p_i)},$$ then winner $i$'s utility can be simplified as $$u_i(c_i, (c_i, {\bf c}_{-i}), {\bf p})=v_i(c_i, p_i)-\frac{\psi^{(2)}({\bf c}, {\bf p})-b_i}{w_i}.$$
		Notice that $\psi^{(2)}({\bf c}, {\bf p})$ is independent of $p_i$, and thereby choosing $p_i$ that maximizes $v_i(c_i, p_i)$ can maximize the winner's utility. Assume advertiser $i$ submits a display price of $\tilde{p}_i^{\mathcal{M}}(c_i)$, we next show that advertiser $i$ cannot increase her utility by choosing other display prices. If $i$ wins the item with the report $(c_i, \tilde{p}_i^{\mathcal{M}}(c_i))$, then according to our previous analysis, the best submitted price is exactly $\tilde{p}_i^{\mathcal{M}}(c_i)$. Otherwise, if $i$ loses the item with the report $(c_i, \tilde{p}_i^{\mathcal{M}}(c_i))$, then she will still lose by arbitrary $p_i$ according to the definitions of $\psi_i$ and $\tilde{p}_i^{\mathcal{M}}(c_i)$. Therefore, her payment and utility are still zero.
		Combined with above analysis, we conclude that for all $i$, all others' reports $({\bf c}_{-i}, {\bf p}_{-i})$, submitting $(c_i, \tilde{p}_i^{\mathcal{M}}(c_i))$ can maximize the utility.
	\end{proof}
	Recall that the allocation policy of the WM-RP is an instance of  the affine maximizer allocation policy, therefore the following result is straightforward.
	\begin{corollary}\label{vw_ef}
		The WM-RP is EF.
	\end{corollary}
	\begin{proof}
		According to the definition of the WM-RP, its allocation policy is an affine maximizer with ${\bf b}=\{0, \cdots, 0\}$ and ${\bf w}=\{1, \cdots, 1\}$. Since the WM-RP is IC and IR for all reported prices, we know that each advertiser will submit a report $(c_i, \tilde{p}_i^{\mathcal{M}}(c_i))$ to maximize her own utility. In other words, in the WM-RP the item will be allocated to the advertiser with the maximum gain $\max_{i\in N}\{v_i(c_i, \tilde{p}_i^{\mathcal{M}}(c_i))\}$, implementing the efficient allocation based on Def. 3 and Proposition \ref{aff_all}.
	\end{proof}
	Corollary \ref{vw_ef} shows that the platform can achieve the maximum social welfare by designing proper auction mechanism, even when the advertisers' strategize on the display prices.
	%Notice that in affine maximizer allocation policies, the equilibrium prices depend on advertisers' cost information, which implies that the formula of the seller's expected revenue characterized in Theorem \ref{expected_p} cannot apply here.
	%, which holds in the non-strategic display price setting, does not hold here. 
	%According to the proof of Proposition \ref{aff_all}, the expected revenue of an auction $\mathcal{M}$ with an affine maximizer allocation policy can be formulated as 
	%$$\mathrm{E}_{{\bf c}\sim \mathcal{C}}[\sum_{i\in N}\pi_i({\bf c}, {\bf \tilde{p}^{\mathcal{M}(c)}})\frac{\psi^{(2)}({\bf c}, {\bf \tilde{p}^{\mathcal{M}(c)}})-b_i}{w_i}].$$
	To optimize the platform's revenue over all affine maximizer allocation policies, we only need to adjust the parameters $({\bf w}, {\bf b})$, which is a typical optimization problem. Since there is no known
	short-cut for calculating the expected revenue \cite{guo2017}, previous studies have developed many techniques to search for the (approximate) optimal parameters, e.g., grid-based gradient descent approach \cite{lik2004,lik2005}, linear programming based heuristic \cite{guo2017}, neutral networks \cite{mic2022}.
	%\subsection{A Tuned Myerson Mechanism}
	%This is a hard problem.
	%{\bf Question:} Could we just place a reserve price for the above mechanism to obtain the optimal revenue? {\bf (just like Myerson did)}
	\section{Conclusion}
	This paper investigates the problem of ad auction design in the presence of display prices. We provide a characterization for all IC,  IR auctions with display prices, and analyze the welfare-maximizing and revenue-maximizing auctions under different scenarios.  Our results show that the engagement of the price information does affect the advertisers' bidding behaviors and the platform can leverage the price information to optimize the performance of ad delivery. There are many other works worthy of further investigation. For example, we only analyze two special classes of allocation policies for the strategic display price setting, a more broad class of mechanisms need be further studied. In addition, though the revenue-maximizing auction in non-strategic display price setting is identified, characterizing the revenue-maximizing auction in the general setting is still an open problem.
	
	\bibliographystyle{named}
	\bibliography{ijcai22}

\begin{thebibliography}{}

\bibitem[\protect\citeauthoryear{Aggarwal \bgroup \em et al.\egroup
  }{2019}]{agg2019}
Gagan Aggarwal, Ashwinkumar Badanidiyuru, and Aranyak Mehta.
\newblock Autobidding with constraints.
\newblock In {\em Web and Internet Economics}, pages 17--30, 2019.

\bibitem[\protect\citeauthoryear{Castiglioni \bgroup \em et al.\egroup
  }{2022}]{Castiglioni2022EfficiencyOA}
Matteo Castiglioni, Diodato Ferraioli, Nicola Gatti, Alberto Marchesi, and
  Giulia Romano.
\newblock Efficiency of ad auctions with price displaying.
\newblock In {\em Proc. of the AAAI Conf. on Artificial Intelligence}, pages
  4933--4940, 2022.

\bibitem[\protect\citeauthoryear{Curry \bgroup \em et al.\egroup
  }{2022}]{mic2022}
Michael~J. Curry, Tuomas Sandholm, and John~P. Dickerson.
\newblock Differentiable economics for randomized affine maximizer auctions.
\newblock {\em CoRR}, abs/2202.02872, 2022.

\bibitem[\protect\citeauthoryear{Edelman and Ostrovsky}{2007}]{EDELMAN2007192}
Benjamin Edelman and Michael Ostrovsky.
\newblock Strategic bidder behavior in sponsored search auctions.
\newblock {\em Decision Support Systems}, 43(1):192--198, 2007.

\bibitem[\protect\citeauthoryear{Edelman \bgroup \em et al.\egroup
  }{2007}]{edelman2007}
Benjamin Edelman, Michael Ostrovsky, and Michael Schwarz.
\newblock Internet advertising and the generalized second-price auction:
  Selling billions of dollars worth of keywords.
\newblock {\em American Economic Review}, 97(1):242--259, 2007.

\bibitem[\protect\citeauthoryear{{Facebook Business Help
  Center}}{2022}]{facebook:auto}
{Facebook Business Help Center}.
\newblock About bid strategies.
\newblock
  \url{https://www.facebook.com/business/help/223852498347426?id=2393014447396453},
  2022.
\newblock [Online; accessed 5-July-2022].

\bibitem[\protect\citeauthoryear{Golrezaei \bgroup \em et al.\egroup
  }{2021}]{glo2021}
Negin Golrezaei, Ilan Lobel, and Renato Paes~Leme.
\newblock Auction design for roi-constrained buyers.
\newblock In {\em Proc. of the Web Conference 2021}, page 3941–3952, 2021.

\bibitem[\protect\citeauthoryear{{Google Ads Help}}{2022}]{google:auto}
{Google Ads Help}.
\newblock About automated bidding.
\newblock \url{https://support.google.com/google-ads/answer/2979071?hl=en},
  2022.
\newblock [Online; accessed 5-July-2022].

\bibitem[\protect\citeauthoryear{Guo \bgroup \em et al.\egroup
  }{2017}]{guo2017}
Mingyu Guo, Hideaki Hata, and Ali Babar.
\newblock Optimizing affine maximizer auctions via linear programming: An
  application to revenue maximizing mechanism design for zero-day exploits
  markets.
\newblock In {\em PRIMA 2017: Principles and Practice of Multi-Agent Systems},
  pages 280--292, 2017.

\bibitem[\protect\citeauthoryear{{IAB}}{2022}]{iab:2021}
{IAB}.
\newblock Internet advertising revenue report.
\newblock
  \url{https://www.iab.com/wp-content/uploads/2022/04/IAB_Internet_Advertising_Revenue_Report_Full_Year_2021.pdf},
  2022.
\newblock [Online; accessed 5-July-2022].

\bibitem[\protect\citeauthoryear{Jansen and Mullen}{2008}]{jansen2008}
Bernard~J. Jansen and Tracy Mullen.
\newblock Sponsored search: an overview of the concept, history, and
  technology.
\newblock {\em International Journal of Electronic Business}, 6(2):114--131,
  2008.

\bibitem[\protect\citeauthoryear{Krishna}{2009}]{krishna2009auction}
Vijay Krishna.
\newblock {\em Auction theory}.
\newblock Academic Press, 2009.

\bibitem[\protect\citeauthoryear{Laffont and Robert}{1996}]{LAFFONT1996181}
Jean-Jacques Laffont and Jacques Robert.
\newblock Optimal auction with financially constrained buyers.
\newblock {\em Economics Letters}, 52(2):181--186, 1996.

\bibitem[\protect\citeauthoryear{Li \bgroup \em et al.\egroup
  }{2019}]{chenchen2019}
Chenchen Li, Xiang Yan, Xiaotie Deng, Yuan Qi, Wei Chu, Le~Song, Junlong Qiao,
  Jianshan He, and Junwu Xiong.
\newblock Latent dirichlet allocation for internet price war.
\newblock In {\em Proc. of the Thirty-Third AAAI Conf. on Artificial
  Intelligence}, pages 639--646, 2019.

\bibitem[\protect\citeauthoryear{Likhodedov and Sandholm}{2004}]{lik2004}
Anton Likhodedov and Tuomas Sandholm.
\newblock Methods for boosting revenue in combinatorial auctions.
\newblock In {\em Proc. of the 19th National Conf. on Artifical Intelligence},
  page 232–237, 2004.

\bibitem[\protect\citeauthoryear{Likhodedov and Sandholm}{2005}]{lik2005}
Anton Likhodedov and Tuomas Sandholm.
\newblock Approximating revenue-maximizing combinatorial auctions.
\newblock In {\em Proc. of the 20th National Conf. on Artificial Intelligence},
  page 267–273, 2005.

\bibitem[\protect\citeauthoryear{{Microsoft Advertising
  Editor}}{2022}]{microsoft:auto}
{Microsoft Advertising Editor}.
\newblock Automatically optimize your campaign with bid strategies.
\newblock \url{https://help.ads.microsoft.com/}, 2022.
\newblock [Online; accessed 5-July-2022].

\bibitem[\protect\citeauthoryear{{Moran} and {Hunt}}{2005}]{moran2005search}
Mike {Moran} and Bill {Hunt}.
\newblock {\em Search Engine Marketing, Inc.: Driving Search Traffic to Your
  Company's Web Site}.
\newblock IBM Press, 2005.

\bibitem[\protect\citeauthoryear{Myerson}{1981}]{myerson1981optimal}
Roger~B Myerson.
\newblock Optimal auction design.
\newblock {\em Mathematics of Operations Research}, 6(1):58--73, 1981.

\bibitem[\protect\citeauthoryear{Naroditskiy and Greenwald}{2007}]{Victor2007}
Victor Naroditskiy and Amy Greenwald.
\newblock Using iterated best-response to find bayes-nash equilibria in
  auctions.
\newblock In {\em Proc. of the 22nd National Conf. on Artificial Intelligence},
  page 1894–1895, 2007.

\bibitem[\protect\citeauthoryear{Shen \bgroup \em et al.\egroup
  }{2017}]{shen2017}
Weiran Shen, Peng Binghui, Hanpeng Liu, Michael Zhang, Ruohan Qian, Yan Hong,
  Zhi Guo, Zongyao Ding, Pengjun Lu, and Pingzhong Tang.
\newblock Reinforcement mechanism design: With applications to dynamic pricing
  in sponsored search auctions.
\newblock In {\em Proc. of the AAAI Conf. on Artificial Intelligence}, pages
  2236--2243, 2017.

\bibitem[\protect\citeauthoryear{Varian}{2007}]{VARIAN20071163}
Hal~R. Varian.
\newblock Position auctions.
\newblock {\em International Journal of Industrial Organization},
  25(6):1163--1178, 2007.

\bibitem[\protect\citeauthoryear{Wilkens \bgroup \em et al.\egroup
  }{2017}]{wilkens2017}
Christopher~A. Wilkens, Ruggiero Cavallo, and Rad Niazadeh.
\newblock Gsp: The cinderella of mechanism design.
\newblock In {\em Proc. of the 26th Int. Conf. on World Wide Web}, page
  25–32, 2017.

\bibitem[\protect\citeauthoryear{{Yang} \bgroup \em et al.\egroup
  }{2019}]{yang2019aiads}
Xiao {Yang}, Daren {Sun}, Ruiwei {Zhu}, Tao {Deng}, Zhi {Guo}, Zongyao {Ding},
  Shouke {Qin}, and Yanfeng {Zhu}.
\newblock Aiads: automated and intelligent advertising system for sponsored
  search.
\newblock In {\em Proc. of the 25th Int. Conf. on Knowledge Discovery and Data
  Mining}, pages 1881--1890, 2019.

\end{thebibliography}
	\appendix

\end{document}